\documentclass[runningheads]{llncs}
\usepackage[T1]{fontenc}
\usepackage{graphicx}
\usepackage{listings}
\usepackage{amsmath}
\usepackage{amsfonts}
\usepackage{amssymb}
\usepackage{listings}
\usepackage{multicol}
\usepackage{todonotes}
\usepackage{subcaption}

\usepackage{tikz}
\usetikzlibrary{shapes,decorations,shadows,arrows,calc}
\makeatletter

\tikzstyle{arg}=[draw,circle,fill=gray!15,inner sep=1pt,minimum size=.5cm]
\tikzstyle{argd}=[draw,circle,fill=gray!15,inner sep=1pt,minimum size=.5cm,dashed]

\tikzstyle{argTD}=[draw, thick, circle, fill=gray!15,inner sep=0pt,minimum size=0.6cm,font=\small]
\tikzstyle{argR}=[draw, thick, circle, fill=gray!15,inner sep=0pt,minimum size=0.45cm,font=\small]
\tikzstyle{scc}=[draw, thick, rectangle,align=center, fill=gray!15,inner sep=0pt,minimum size=0.8cm,font=\small, rounded corners=0ex,]
\tikzstyle{argTDX}=[draw, dotted,thick, circle, inner sep=0pt,minimum size=0.6cm,font=\small]
\tikzstyle{sccX}=[draw,dotted, thick, rectangle,align=center, inner sep=0pt,minimum size=0.8cm,font=\small, rounded corners=0ex,]
\tikzstyle{argsmall}=[draw, thick, circle, fill=gray!15,inner sep=0pt,minimum size=0.4cm]
\tikzstyle{argsmallX}=[draw, thick, circle, inner sep=0pt,minimum size=0.3cm,dotted]
\tikzstyle{bag}=[draw,rectangle, rounded corners=1ex,
align=center, inner sep=1ex]

\tikzstyle{atts}=[draw,thick, inner sep=5pt, rounded corners=3pt]

\tikzstyle{nullarg}=[inner sep=0pt,outer sep=0pt,minimum size=0cm]
\tikzstyle{nullattack}=[draw, thick, |->]

\newcommand{\naf}{\ensuremath{\mathtt{not}\,}}

\newcommand{\cf}{\textit{cf}}

\newcommand{\stb}{\textit{stb}}

	\makeatletter
	\newcommand{\squplus}{\mathbin{\mathpalette\squplus@{}}}
	
	\newcommand{\squplus@}[2]{%
\begingroup
\sbox0{$#1\sqcup$}%
\setlength{\unitlength}{1ex}%
\raisebox{0.29\ht0}{%
	\rlap{%
		\kern0.18\wd0%
		\scalebox{0.48}{$#1\boldsymbol{+}$}%
	}%
}%
\sqcup
\endgroup
}
\makeatother

\usepackage{thm-restate}
\usepackage{etoolbox}

\renewcommand\thmcontinues[1]{ctd}

\begin{document}
\title{On Strong Equivalence Notions in Logic Programming and Abstract Argumentation}
\titlerunning{On Strong Equivalence in Logic Programming and Abstract Argumentation}

\author{Giovanni Buraglio\orcidID{0009-0004-9592-4739} \and
Wolfgang Dvo\v{r}\'ak\orcidID{0000-0002-2269-8193} \and
Stefan Woltran\orcidID{0000-0003-1594-8972}}
\authorrunning{G. Buraglio et al.}
\institute{TU Wien, Austria \\
\email{\{giovanni.buraglio, wolfgang.dvorak, stefan.woltran\}@tuwien.ac.at}}

\maketitle              
\begin{abstract}
Strong equivalence between knowledge bases ensures the possibility of replacing one with the other without affecting reasoning outcomes, in any given context. This makes it a crucial property in nonmonotonic formalisms. In particular, the fields of logic programming and abstract argumentation provide primary examples in which this property has been subject to vast investigations. However, while (classes of) logic programs and abstract argumentation frameworks are known to be semantically equivalent in static settings, this alignment breaks in dynamic contexts due to differing notions of update. As a result, strong equivalence does not always carry over from one formalism to the other.
In this paper, we carefully investigate this discrepancy and introduce a new notion of strong equivalence for logic programs. Our approach preserves strong equivalence under translation between certain classes of logic programs and both Dung-style and claim-augmented argumentation frameworks, thus restoring compatibility across these formalisms.

\end{abstract}

\section{Introduction}
In the field of Knowledge Representation and Reasoning, the concept of equivalence between knowledge bases has been the subject of extensive studies. A primary motivation behind this research is the potential to exploit the equivalence between two knowledge bases to achieve a compact representation of the same information. Furthermore, the possibility to substitute a specific component of a given knowledge base $\mathcal{K}$ with an equivalent yet simplified alternative has been shown to ease the computational cost of reasoning over $\mathcal{K}$.
While this advantageous behavior can be taken for granted in monotonic formalisms (e.g. classical logic), it is usually not the case in non-monotonic ones. For this, the notion of strong equivalence has been introduced in a dynamic environment, to capture the idea of equivalence under any possible update \cite{Cabalar02,EiterFW07,LifschitzPV01,OikarinenW11,Truszczynski06,Turner01}.

In this paper, we consider two families of non-monotonic formalisms, namely logic programming and abstract argumentation. Logic programming is a declarative programming paradigm where a reasoning problem is specified by means of a so-called logic program (LP). This consists of a set of inference rules made of atoms, possibly preceded by a negation-as-failure operator. Negative literals are assumed to be true as long as their corresponding positive atom cannot be proven to hold. Abstract argumentation is a sub-field of symbolic Artificial Intelligence~\cite{arguHandbookv1} that offers formal approaches to represent and reason over situations where conflicting information is present. An argumentative scenario is specified by means of an abstract argumentation framework, which is a directed graph where nodes represent arguments and edges their relation. The starting point in the field is the seminal work of Dung~\cite{Dung95}, where argumentation frameworks, which we call Dung-style AFs, contain only a relation of attack. 
Several semantics have been proposed for both of these formalisms, with the common purpose of extracting solutions for the given program or argumentation framework. These are respectively sets of atoms that satisfy each rule in the program (called \textit{answer-sets}), and sets of arguments that are able to defend themselves against possible counter-arguments (called \textit{extensions}). In this work, we focus on the stable model semantics for LPs \cite{GelfondL88} and the stable semantics for AFs~\cite{Dung95}.

Previous work has introduced a one-to-one mapping between Dung-style AFs and (a resticted class of) LPs under the stable model semantics~\cite{CaminadaSAD15a,Dung95}. Any problem can be either specified via a program $P$ or an abstract argumentation framework $F$ in such a way that their solutions coincide. 
Later, this mapping has been extended to a wider class of LPs and generalizations of Dung-style AFs, such as claim-augmented AFs~\cite{DvorakW20}. 
Thus, equivalent LPs correspond to equivalent AFs. 

However, such a semantic correspondence does not necessarily carry over to dynamic contexts, where strong equivalence is required. It is possible for two logic programs that are strongly equivalent to induce argumentation frameworks that are not, due to the incongruous notions of update (or expansion) in the two realms. To acknowledge such mismatch, consider the following example:
\begin{example}
    Two suspects, X and Y, are under investigation for a murder proven by forensic analysis to have been committed by a single person. 
	While questioning possible witnesses, the detective learns that X and Y have been seen far away from the crime scene at the time of the murder, and the case remains unsolved. 
    We model the detective's knowledge via an Dung-style AF and an LP as follows:
	
	\vspace{1mm}
    \hspace{-20pt}
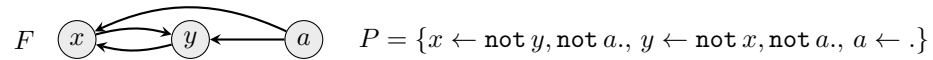
	\begin{minipage}{.37\textwidth}
		\begin{tikzpicture}
				\path
			(-1.7,0) node(F){$F$}
			(-1,0) node[arg](x){$x$}
			(0.5,0) node[arg](y){$y$}
			(2,0) node[arg] (a) {$a$}
			(0,-.45) node (v) {}
			;
			
			\path[->,>=stealth,thick]
			(x) edge[bend left=15] (y)
			(y) edge[bend left=15] (x)
			(a) edge[bend right=25] (x)
			(a) edge (y)
			;		
		\end{tikzpicture}
	\end{minipage}
	\begin{minipage}{.6\textwidth}
	 	\begin{tabular}{cl}
	 		$P=$& $ \{x\gets \naf y, \naf a.,$ 
	 		  $ y\gets \naf x, \naf a.,$ 
	 		  $ a \gets .\}$ 
	 	\end{tabular}
	  \end{minipage}	 

	where $a$ means ``has an alibi" and $x$ (resp. $y$) means ``X (resp. Y) is the murder". 
	Later, another witness testifies that Z was drunk ($d$) on the same night, falsifying the suspects' alibi. 
    This information update is captured via:
	
	\vspace{2mm}
	\begin{minipage}{.45\textwidth}
		\begin{tikzpicture}
			\path
			(-1,0) node (F') {$F'$}
			(0,0) node[arg](a){$a$}
			(1.5,0) node[arg](d){$d$}
			;
			
			\path[->,>=stealth,thick]
			(d)edge(a)
			;
		\end{tikzpicture}
	\end{minipage}
	\begin{minipage}{.55\textwidth}
$P'=\{a\gets \naf d., \; d \gets .\}$
	\end{minipage}	 

\vspace{2mm}

	In accordance with equivalence results between logic programs and argumentation frameworks, the two ways of modeling our knowledge base are consistent with each other when taken individually. However, this does not happen when they are combined. Incorporating the second pair of knowledge bases ($F'$ and $P'$) into the first one (resp. $F$ and $P$) yields different result: in the case of $F$ and $F'$, their union returns the expected result in the form of two possible extensions $\{d,x\}$ and $\{d,y\}$. Indeed, since the witness Z was drunk, this is compatible with either X or Y being the murder. On the other hand, the union of the two logic programs $P$ and $P'$ yields an unexpected prediction: the alibi `$a\gets $' is not overwritten by Z being drunk `$d\gets$', leaving $\{a,d\}$ as the only possible solution.
\end{example}

Updating an existing program by simply adding rules may yield unexpected outcomes. Facts (e.g. `$a\gets $' in $P$) cannot be overwritten by incoming information, while their corresponding arguments (e.g. $a\in F$) can. This behavior of LPs is fundamentally in contrast with the way in which non-monotonicity is encoded in abstract argumentation: an argument can always be attacked by new ones.  

In this work, we carefully analyze the relationship between logic programming and abstract argumentation, with a particular interest in dynamic contexts. As a first step, we recall and extend equivalence results for classes of logic programs and abstract argumentation frameworks. Subsequently, motivated by the mismatch above, we introduce a novel notion of update for a restricted class of LPs, called Rule Refinement, that resolves the issue by mimicking precisely the existing notion for Dung-style AFs. We further extend Rule Refinement to the wider class of atomic logic programs (where no positive literal occurs in the body), which is shown to capture strong equivalence in claim-augmented AFs. 

\section{Preliminaries}
\paragraph{Logic Programming}
We consider normal logic programs with negation-as-failure $\mathtt{not}$, consisting of rules $r$ of the form
‘$n : c \gets a_1, \dots , a_k, \mathtt{not}\, b_1, \dots , \mathtt{not}\, b_m.$'
read as ‘$c$ if $a_1$ and $\dots$ and $a_k$ and $\mathtt{not}\, b_1$ and $\dots$ and $\mathtt{not}\, b_m$’.
Here, $n\in \mathbb{N}$ is the \textit{identifier} of the rule $r$ in the program; we refer to it with $id(r)=n$. Further, $a_i$, $b_i$ and $c$ are ordinary atoms; $\mathcal{L}(P)$ is the set of all atoms occurring in $P$. The atoms $a_i$ and $b_j$ are called positive and negative atoms, respectively denoted by $pos(r)=\{a_1, \dots, a_k\}$ and $neg(r)=\{b_1,\dots,b_m\}$. We use $head(r)=c$ and $body(r) = \{a_1,\dots , a_k, \mathtt{not}\, b_1, \dots, \mathtt{not}\, b_m\}$ for the head and body of $r$. With a slight abuse of notation, we extend $pos$, $neg$, $head$, $body$ to sets of rules, e.g.\ $head(R)=\{head(r)\mid r\in R\}$. For a set of atoms $S$, we use $\mathtt{not}\, S=\{\mathtt{not}\, b \mid b \in S\}$. 
A rule $r$ is called: \textit{constraint} if $ head(r)=\emptyset$; \textit{fact} if $k=m=0$; \textit{atomic} if $k=0$.
A program $P$ is \textit{atomic} if all the rules in $P$ are atomic. 
The stable model semantics of normal LPs is defined in terms of answer sets~\cite{GelfondL88}. 

\begin{definition}
	Let $P$ be a normal logic program and $S\subseteq \mathcal{L}(P)$ a set of atoms. The \emph{reduct} of $P$ w.r.t. $S$ is the negation-free program $P^S$ obtained from $P$ by: (i) deleting all rules $r \in P$ with $\naf b_j$ in the body for some $b_j \in S$, (ii)  deleting all negated atoms from the remaining rules.
	$S$ is an answer set of $P$ iff $S$ is the minimal model of $P^S$. The collection of answer sets of a program $P$ is $AS(P)$.
\end{definition} 

In this paper, we focus on the class of atomic LPs. Whenever we write program we mean an atomic one. We further restrict our attention to sub-classes of atomic programs: we call \emph{strict} (resp.\ \emph{h-unique}) a logic program $P$ where each atom $p\in \mathcal{L}(P)$ occurs at least (resp. at most) once in the head of a rule.

\paragraph{Abstract Argumentation}

We  fix  a infinite  background  set of arguments $\mathcal{U}$. A strict argumentation framework (strict AF\footnote{We use this terminology for standard Dung AFs, to emphasize the fact that the attack relation is restricted to arguments in $A$. In Section~\ref{sec: translation}, we will relax this requirement.}) is a directed graph $F=(A,R)$ where $A\subseteq \mathcal{U}$ is a finite set of arguments and $R\subseteq A\times A$ an attack relation between them. The union of any two strict AFs $F= (A,R)$ and $G= (A',R')$ is defined as $F\cup G=(A\cup A',R\cup R')$. For  two  arguments $a,b\in A$ we say that $a$ attacks $b$ if $(a,b)\in R$. Moreover, a set of arguments $E\subseteq A$ attacks $b$ if $(a,b)\in R$ for some $a\in E$ and $a$ attacks $E$ if $(a,b)\in R$ for some $b\in E$.
We use $E^+_R=\{a\in A\mid E \text{ attacks } a\}$ and $E^-_R=\{a\in A\mid a \text{ attacks } E\}$ to denote the set of arguments respectively attacked by and attacking $E$. Further, $E^\oplus_R=E \cup E^+_R$ is called the \emph{range} of $E$ w.r.t. $R$. For a singleton $\{a\}$ we use $a^+_R$, $a^-_R$ and $a^{\oplus}_R$. 

$E$ is conflict-free in $F$ ($E\in \cf(F)$) iff for no $a,b\in E$, $(a,b)\in R$. 
Several semantics for abstract argumentation have been introduced~\cite{Dung95}. These are functions $\sigma: F \mapsto \sigma(F)\subseteq 2^A$, that return a set of subsets of $A$ (called $\sigma$-extensions) for any AF $F= (A,R)$. Here we consider only \textit{stable} semantics ($\stb$, for short): 
\begin{definition}\label{def: af stb}
	Let $F=(A,R)$ be an AF and $E \in \cf(F)$ a conflict-free set in $F$. We say that $E$ is a \textit{stable} extension of $F$ $(E \in \stb(F))$ iff $E^\oplus_R=A$.
\end{definition}

In recent years, more expressive abstract formalisms have been proposed that extend Dung-style AFs. Among these, \textit{claim-augmented argumentation frameworks} (or CAFs) add claims to the abstract representation \cite{DvorakRW20,DvorakW20}. 

\begin{definition}
	Let $\mathcal{C}$ be a set (or universe) of claims. A claim-augmented argumentation framework (CAF) is a triple $\mathcal{F}=
	(A,R,\gamma)$ such that $F=(A,R)$ is a strict AF and $\gamma: A\mapsto \mathcal{C}$ is a function that assigns claims to arguments. 
\end{definition} 
For a set of arguments $E$, we use $\gamma(E)=\{\gamma(a)\mid a\in E\}$ to denote the associated set of claims. 
By means of CAFs, it is possible to represent situations where different arguments have the same claim. 
In this paper we focus \textit{well-formed CAFs}, for which arguments with same claim attack the same arguments. A CAF $\mathcal{F}=(A,R,\gamma)$ is well-formed iff for any two arguments $a$, $b$, $\gamma(a)=\gamma(b)$ implies $a^+_R=b^+_R$. 
Stable semantics for a well-formed CAF $\mathcal{F}=(F,\gamma)$ can be defined as taking the claim-sets $\gamma(E)$ inherited from every stable extension $E$.
\begin{definition}
	   For a well-formed CAF $\mathcal{F}=(F,\gamma)$, a set $S$ is a stable (claim-)extension of $\mathcal{F}$ $(S \in \stb_i(\mathcal{F}))$ iff there is an $E\in \stb(F)$ such that $S=\gamma(E)$.
\end{definition} 

\paragraph{Strong Equivalence}
Strong equivalence notions for logic programs and argumentation frameworks have been presented to capture equivalence under any possible updates. 
In the case of LPs, updates consist of expanding the original program with a set of rules. 
Two LPs $P$ and $Q$ are said to be equivalent, denoted $P \equiv Q$ whenever $AS(P) = AS(Q)$. Further, $P$ and $Q$ are strongly equivalent, denoted $P \equiv_s Q$, iff, for any LP $R$, the programs $P \cup R$ and $Q \cup R$ are equivalent, i.e. $P \cup R \equiv Q \cup R$ \cite{LifschitzPV01}. 
Here, we consider specific classes of LPs and adjust the relative notion of strong equivalence accordingly. 
\begin{definition}
	Given a class $\Pi$ of LPs, $P$ and $Q$ in $\Pi$, we say that $P$ and $Q$ are strongly equivalent in $\Pi$, written $P \equiv^\Pi_s Q$, iff for every program $R$: either (1) $P\cup R\notin \Pi$ and $Q\cup R \notin \Pi$ or (2) $P\cup R\equiv Q\cup R$. 
\end{definition}

More recently, strong equivalence has been studied for other non-monotonic formalisms, in the field of abstract argumentation \cite{OikarinenW11}. Two strict AFs $F$ and $G$ are equivalent under stable semantics, if $\stb(F)=\stb(G)$. Strong equivalence, denoted by $F \equiv_s G$, is satisfied whenever $\stb(F\cup H)=\stb(G\cup H)$ for any strict AF $H$. 
In this setting, strong equivalence admits a syntactic characterization in terms of so-called (semantics-dependent) kernels, obtained by syntactical modifications of the given frameworks. For stable semantics, such modification consists in the removal of out-going attacks from self-attacking arguments. 

\begin{definition}[\cite{OikarinenW11}]
	Let $F = (A,R)$ be an strict AF. The stable kernel of $F$ is $F^{SK}=(A,R^{SK})$ with $R^{SK}=R \setminus \{(a,b) \mid a \neq b, (a,a) \in R\}$.
\end{definition}
A characterization of strong equivalence can be obtained via kernels~\cite{OikarinenW11}.
\begin{proposition}[\cite{OikarinenW11}]\label{pro:AF_SE}
	For any strict AFs $F$ and $G$, $F \equiv_s G$ iff $F^{SK}=G^{SK}$.
\end{proposition}
The idea of strong equivalence has been investigated for CAFs as well. In the literature only expansions among \textit{compatible} CAFs are considered. Informally, two CAFs are compatible if and only if the arguments they share have the same claim. Moreover, AF kernels characterize strong equivalence for CAFs as well. 

\begin{proposition}[\cite{BaumannRU23}]
	Let $\mathcal{F}=(F,\gamma)$ and $\mathcal{G}=(G,\gamma')$ be two compatible well-formed CAFs. Then $ \mathcal{F} \equiv^{\stb_i}_s \mathcal{G} \text{ iff } F\equiv^{\stb}_s G$.
\end{proposition}

This kernel characterization applies when $\mathcal{F}$ and $\mathcal{G}$ are well-formed, for any common update $\mathcal{H}$ (possibly not well-formed). In Section \ref{subsec:SELPsandWFCAFs} we follow a different path, since CAFs translated from LPs are guaranteed to be well-formed and the requirement of compatibility for expansions is not natural for logic programs.  

\section{From Abstract Argumentation to Logic Programming and Back}\label{sec: translation}
In this section we introduce translations between increasingly larger classes of logic programs and increasingly more expressive abstract argumentation frameworks. We recall the fact that strict h-unique atomic programs correspond to strict AFs~\cite{CaminadaSAD15a}. Further, we show that by relaxing the strictness requirement for LPs we end up in the broader class of (possibly non-strict) AFs. Then, we consider the full class of atomic LPs and their correspondence to well-formed CAFs. 

\subsection{Logic Programs and Dung-style AFs}
We begin with strict h-unique atomic programs and strict AFs.
Any strict AF can be transformed into an logic program by generating a rule $r$ for every argument such that the head of $r$ is the argument's name and the (negative) body contains the names of each attacker. The resulting program is atomic, strict and h-unique. 
\begin{definition}[\cite{CaminadaSAD15a}]\label{def: AF2LP}
	Let $F = (A,R)$ be an argumentation framework with $A=\{a_1,\dots, a_n\}$. The corresponding logic program of $F$ is:
	\begin{align*}
		P_F =\{& i: a_i \gets \naf b_1,\dots,\naf b_k. \mid a_i \in A \text{ and } \{b_1,\dots,b_k\}=a_i^-\}
	\end{align*}
\end{definition}

\begin{example}\label{ex:AF2LP}
	Consider the following strict AF $F=(A,R)$ depicted below with $A=\{a,b,c\}$ and $R=\{(a,a),(a,b),(b,a),(b,c)\}$ and its corresponding LP $P_F$:

    \hspace{-20pt}
	
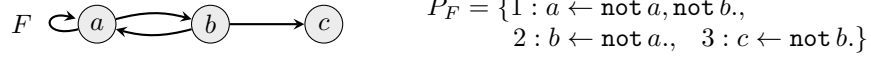
	\begin{minipage}{.45\textwidth}
		\begin{tikzpicture}
			\path
			(-2.5,0) node (F) {$F$}
			(0,0) node[arg](b){$b$}
			(-1.5,0) node[arg](a){$a$}
			(1.5,0) node[arg](c){$c$}
			;
			
			\path[->,>=stealth,thick]
			(a)edge[loop left](a)
			(a)edge[bend left=15](b)
			(b)edge[bend left=15](a)
			(b)edge(c)
			;
		\end{tikzpicture}
	\end{minipage}
	\begin{minipage}{.65\textwidth}
	\begin{tabular}{cl}
		$P_F =$& $\{ 1: a\gets \naf a, \naf b.,$ \\
		& \; $2: b\gets \naf a.,$ 
		 \; $3: c \gets \naf b.\}$ \\
		\end{tabular}
	\end{minipage}
\end{example}

For the other direction, the transformation only takes strict h-unique programs as input. For each rule, the head atom is associated to an argument with the same name and each negated atom (called  \textit{vulnerability}) in the body is associated with an attacker of the argument corresponding to the rule-head. 

\begin{definition}[\cite{CaminadaSAD15a}]\label{def:LP2AF}
	Let $P$ be a strict h-unique program. The corresponding strict argumentation framework $F_P=(A_P,R_P)$ is defined by:
	\begin{itemize}
		\item $A_P=\{a \mid a \in head(r), \, r \in P\}$,
		\item $R_P=\{(a, b) \mid b \in head(r), \, a \in neg(r) , \, r \in P\}$.
	\end{itemize}
\end{definition} 

It is easy to see that under translation, the program in Example \ref{ex:AF2LP} generates exactly the strict AF $F$. Indeed, the presented translation is a bijection.

\begin{proposition}[\cite{CaminadaSAD15a}]\label{pro:iso AF-LPs}
For a strict AF $F$ and its corresponding LP $P_F$, $F = F_{P_F}$.
\end{proposition}

Moreover, the answer sets for a (strict h-unique) program correspond exactly to the stable extensions of its transformed strict AF.

\begin{proposition}[\cite{CaminadaSAD15a}]\label{prop:equiv AF LP}
	For any strict AF $F$ and h-unique atomic program $P$, $\stb(F)=AS(P_{F})$ and $AS(P)=\stb(F_P)$.
\end{proposition}

Given a certain program $P$, the attacks that are generated in $F_P$ depends on two factors: the vulnerabilities and head-atoms. In \cite{CaminadaSAD15a} equivalence under transformation is studied in a static setting, where strictness becomes a crucial aspect to enforce in order to define a one-to-one mapping. 
In a dynamic setting, on the other hand, the vulnerabilities in $P$ which do not occur as head-atoms may have an impact at the semantic level: new incoming attacks in $F_P$ may arise.  

\begin{example}
	The programs $P=\{a\gets \naf b.\}$ and $P'=\{a\gets. \}$ have the same unique answer set $\{a\}$. However, when concurrently updated with $P''=\{b\gets.\}$, their answer sets differ: $AS(P\cup P'')=\{\{b\}\}$ and $AS(P'\cup P'')=\{\{a,b\}\}$. 
\end{example}

Motivated by these considerations, we relax the notion of strictness for logic program, i.e. we consider programs with negative literals that never occur as head-atoms. Indeed, the previously shown translation would relate non-strict programs to strict AFs only at the cost of losing the one-to-one mapping: several non-strict LPs correspond to the same strict AF. In order to maintain an exact mapping, we also relax the definition of attack relation, allowing incoming attacks from elements in the universe $\mathcal{U}$ outside of $A$. We use \textit{ungrounded attacks} to refer to these relations. Further, we simply call argumentation framework (AF) any framework with possibly ungrounded attacks.

\begin{definition}[AF]
	An AF $F=(A,R)$ is a directed graph with $A\subseteq \mathcal{U}$ a non-empty set of arguments and $R\subseteq \mathcal{U}\times A$ a conflict relation. 
    We call $R_{\downarrow A}= R\cap (A\times A)$ the set of (proper) attacks, and $R\setminus R_{\downarrow A}$ the set of ungrounded attacks. 
\end{definition} 

The strictness notion of logic programs and argumentation frameworks are indeed connected: every non-strict AF can be mapped to exactly one non-strict LP. The isomorphism of Proposition \ref{pro:iso AF-LPs} between strict h-unique atomic LPs and strict AFs can be lifted to possibly non-strict programs and AFs.
\begin{remark}\label{rem: iso quasi-AF LP}
	For any AF $F$ and corresponding logic program $P_F$, $F = F_{P_F}$. Conversely, for any h-unique atomic LP $P$ and corresponding AF $F$, $P=P_{F_P}$. 
\end{remark}

\begin{example}
Let $F=(A,R)$ be the AF below. In the program $P_F$, the ungrounded attack $(c,b)\in R\setminus R_{\downarrow A}$ is represented by the vulnerability $c\in neg(P_F)$. 

\vspace{1.5mm}
\hspace{-5pt}
	\begin{minipage}{.35\textwidth}
    \hspace{-18pt}
			\begin{tikzpicture}
				\path
				(-3.5,0) node (F) {$F$}
				(-1,0) node[arg](b){$b$}
				(-2.5,0) node[arg](a){$a$}
				(0.5,0) node[argd](c){$c$}
				;
				
				\path[->,>=stealth,thick]
				(a)edge[loop left](a)
				(a)edge[bend left=15](b)
				(b)edge[bend left=15](a)
				(c)edge[dashed](b)
				;
			\end{tikzpicture}
	\end{minipage}
	\begin{minipage}{.65\textwidth}
$P_F =\{ 1: a\gets \naf a, \naf b., \, 2: b\gets \naf a, \naf c.\}$
	\end{minipage}
	
\end{example}

The notions of defense, conflict-freeness as well as stable semantics for AFs are defined over $R_{\downarrow A}$ and are thus identical to those introduced for strict AFs.

\begin{definition}\label{def:quasi-af stb}
	Let $F=(A,R)$ be an AF. A set $E\subseteq A$ is conflict-free in $F$ iff for no $a,b \in A$, $(a,b)\in R_{\downarrow A}$. Further, $E$ is a stable extension of $F$ iff $E^\oplus_{R_{\downarrow A}}=A$.
\end{definition}
Notice that the notion of stable-kernel characterizes strong equivalence between AFs. This follows straightforwardly from the fact that we fixed $R\subseteq \mathcal{U}\times A$. Hence, no ungrounded attack generates from a self-attacking argument. 
Moreover, ungrounded attacks have no influence on the evaluation of an AF. 

\begin{proposition}\label{prop:equiv AF quasi AF}
	For an AF $F=(A,R)$ and $F_{\downarrow A}=(A,R_{\downarrow A})$, $\stb(F)=\stb(F_{\downarrow A})$.
\end{proposition}

For h-unique programs we observe a similar behavior. We can identify a set $S\subseteq neg(P)$ of \textit{ungrounded vulnerabilities} that make a program $P$ non-strict. Formally, the set of ungrounded vulnerabilities for $P$ is $UV(P)=\{\naf b \mid b\in neg(P) \setminus head(P)\}$. By removing such vulnerabilities, we obtain the strict program $P_{\downarrow head}$, defined as $P_{\downarrow head}=\{id(r): head(r)\gets body(r)\setminus UV(P). \mid r\in P\}$.
Notice that transforming a non-strict atomic program into the corresponding strict one yields equivalent answer-sets, and vice-versa. Ungrounded vulnerabilities have no semantic impact, similarly to ungrounded attacks. 
Thus, answer-sets are preserved when restricting to strict programs. 

\begin{restatable}{proposition}{ASEqualPToWellFormP}\label{prop:equiv LP ill-formed LP}
	For an atomic program $P$ and $P_{\downarrow head}$, $AS(P)=AS(P_{\downarrow head})$.
\end{restatable} 

To show that semantic equivalence can be lifted to AFs and h-unique programs, we introduce the following lemma.

\begin{restatable}{lemma}{OrderTranslQuasi}
	For any h-unique atomic logic program $P$, $F_{P_{\downarrow head}}=(F_P)_{\downarrow A_P}$. For any AF $F$, $P_{F_{\downarrow A}}=(P_F)_{\downarrow head}$. 
\end{restatable}

\begin{restatable}{proposition}{SemCorrAFLP}
	For any h-unique atomic program $P$ and the corresponding AF $F_P$, it holds that $AS(P) = \stb(F_P)$. Conversely, for any AF $F$ and its corresponding program $P_F$, we have that $\stb(F) = AS(P_F)$.
\end{restatable}

\subsection{Logic Programs and well-formed CAFs}
We can now move one step forward and consider the wider class of atomic LPs and their relation to well-formed CAFs. For these, attacks are not arbitrary pairs of arguments, but depends on the attacker's claim. A translation has been introduced between \textit{strict} atomic LPs and well-formed CAFs. In our setting, we provide a more general translation that considers possibly non-strict LPs. To do so, we give a definition of well-formed CAFs more suitable for a dynamic setting.

\begin{definition}
	A well-formed CAF is a triple $\mathcal{F}=(A, R^\mathcal{C}, \gamma)$ where $A$ is a finite set of arguments, $R^\mathcal{C}\subseteq \mathcal{C}\times A$ is a set of claim-attacks with $\mathcal{C}$ the universe of claims and $\gamma:A\mapsto \mathcal{C}$ the claim function.
\end{definition}
One can easily retrieve the original definition of CAF by fixing $R=\{(x_i,x_j)\mid x_i,x_j\in A, \gamma(x_i)=c \text{ and } (c,x_j)\in R^{\mathcal{C}}\}$. As a result, we can utilize the definition of stable semantics for well-formed CAFs by means of this translation.
By appealing to the universe of claims $\mathcal{C}$, we have a formulation that encompasses claim attacks from claims which do not label any argument, i.e. ungrounded attacks. Any atomic LP exactly corresponds to some well-formed CAF by taking an argument for each rule, a set of claims for each atom occurring as the head of some rule and attacks defined from claims to arguments.

\begin{definition}
	Given an atomic program $P$, we obtain the corresponding well-formed CAF $\mathcal{F}_P=(A_P,R^{\mathcal{C}}_P,\gamma_P)$ by fixing $A_P=\{x_i\mid id(r)=i, r\in P\}$, $R^{\mathcal{C}}_P=\{(c,x_i)\mid c\in neg(r), id(r)=i\}$, and $\gamma_P(x_i)=head(r)$ with $id(r)=i$. 
\end{definition}

\begin{example}
	Consider the atomic LP $P$ below and its corresponding well-formed CAF. Notice that the claim-attack $(c,x_2)$ is realized by an ungrounded attack.

    \hspace{-20pt}
	\begin{minipage}{.51\textwidth}
		\begin{tabular}{cl}
			$P =$& $\{ 1: a\gets \naf b.,$ \\
			& \, $2: a\gets \naf c.,$
			       $3: b \gets \naf a.\}$ \\
		\end{tabular}
	\end{minipage}
	\begin{minipage}{.4\textwidth}
		\begin{tikzpicture}
		\path
		(-3.3,0.2) node(F){$\mathcal{F}_{P}$}
		(-2.5,0.6) node (a1) {$a$}
		(-2.5,0.2) node[arg] (1) {$x_1$}
		(-1,0.65) node (b) {$b$}
		(-1,0.2) node[arg](3){$x_3$}
		(0.5,0.6) node (a2) {$a$}
		(0.5,0.2) node[arg](2){$x_2$}
		(2,0.6) node (c) {$c$}
		(2,0.2) node[argd](star){$*$}
		(0,-.25) node (v) {}
		;
		
		\path[->,>=stealth,thick]
		(3) edge[bend left=15] (1)
		(1) edge[bend left=15] (3)
		(2) edge (3)
		(star) edge[dashed] (2)
		;
	\end{tikzpicture}
\end{minipage}
\end{example}
From a well-formed CAF, each argument $x_i$ corresponds to a rule r with $id(r)=i$ and head $\gamma(x_i)$. Further, for each claim $c\in \mathcal{C}$ attacking the argument $x_i$, $\naf c$ appears in the body of the rule. 

\begin{definition}
    For a well-formed CAF $\mathcal{F} = (A,R^{\mathcal{C}},\gamma)$ with $A=\{x_1\dots,x_n\}$, $P_\mathcal{F}=\{i: \gamma(x_i) \gets \naf b_1\dots\naf b_k. \mid x_i \in A, \{b_1\dots b_k\}=\gamma(x_i^-)\}$ is the corresponding LP.  
\end{definition}
Any CAF $\mathcal{F}$ where different arguments share the same claim is associated with a non h-unique program $P_\mathcal{F}$. 
Thus, CAFs are isomorphic to atomic LPs. 

\begin{proposition}[\cite{KoenigRU22}]\label{pro:isoCAF}
    For a CAF $\mathcal{F}$ and corresponding LP $P_\mathcal{F}$, $\mathcal{F}=\mathcal{F}_{P_\mathcal{F}}$. 
\end{proposition}

For well-formed CAFs and LPs, equivalence is preserved under translation. 

\begin{proposition}[\cite{KoenigRU22}]\label{pro:CAF_LP_Corresp}
	For any well-formed CAF $\mathcal{F}$ and atomic program $P$, $\stb(\mathcal{F})=AS(P_{\mathcal{F}})$ and $AS(P)=\stb(\mathcal{F}_P)$.
\end{proposition}

In their original formulation, the propositions above consider only strict well-formed CAFs and atomic programs. However, they can be carefully adapted to our context based on the our relaxed notion of well-formed CAF.

\subsection{Equivalence from Static to Dynamic Setting}

Until now, we have considered equivalence within a static setting and provided translations that preserve it. 
Such a semantic correspondence does not carry over to dynamic scenarios, already within the class of strict AFs.  
Two strongly equivalent LPs may have corresponding AFs that are not strongly equivalent.

\begin{proposition}
	For two LPs $P$ and $Q$, $P \equiv_s Q$ does not imply $F_{P} \equiv_s F_{Q}$.
\end{proposition}

\begin{example}
	Take two logic programs as follows:
	
	\vspace{1mm}
	\begin{minipage}{.53\textwidth}
		\begin{tabular}{cl}
			$P=$&$\{ a\gets \naf b, \naf c.,$ \\
      		&\, $b\gets \naf a, \naf c.,$
			    $c \gets .\}$
		\end{tabular}
	\end{minipage}
	\begin{minipage}{.6\textwidth}
		\begin{tabular}{cl}
			$Q=$& $\{ a\gets \naf c.,$ \\
			& \, $ b\gets \naf a, \naf c.,$
                 $c \gets .\}$
		\end{tabular}
	\end{minipage}
	
	\vspace{1mm}
	Since both $P$ and $Q$ contain the fact $c\gets $ and $c\in neg(r)$ for every other rule in both programs, $c$ occurs in any answer-set of $P\cup R$ and $Q\cup R$ for any $R$. Hence, they are strongly equivalent. 
	The associated strict AFs are: 
	
	\begin{minipage}{.52\textwidth}
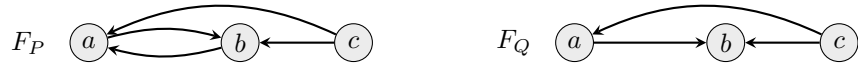
		\begin{tikzpicture}
			\path
			(-1.8,0) node(F){$F_P$}
			(-1,0) node[arg](a){$a$}
			(1,0) node[arg](b){$b$}
			(2.5,0) node[arg] (c) {$c$}
		    (0.9,0.65) node (v) {}
		    (0.9,-0.35) node (v) {}
			;
			
			\path[->,>=stealth,thick]
			(a) edge[bend left=15] (b)
			(b) edge[bend left=15] (a)
			(c) edge[bend right=25] (a)
			(c) edge (b)
			;
		\end{tikzpicture}
	\end{minipage}
	\begin{minipage}{.55\textwidth}
		\begin{tikzpicture}
		\path
		(-1.8,0) node(F){$F_Q$}
		(-1,0) node[arg](a){$a$}
		(1,0) node[arg](b){$b$}
		(2.5,0) node[arg] (c) {$c$}
		(0.9,0.65) node (v) {}
		(0.9,-0.35) node (v) {}
		;
		
		\path[->,>=stealth,thick]
		(a) edge (b)
		(c) edge[bend right=25] (a)
		(c) edge (b)
		;
		\end{tikzpicture}
	\end{minipage}
	
	Both $F_Q$ and $F_P$ coincide with their own kernels, as they contain no self-attacking argument. Since they are different, they are not strongly equivalent.
\end{example}
		
\section{Strong Equivalence under Rule Refinement}\label{sec: RR}

As previously shown, equivalence among LPs and AFs does not carry over to strong equivalence: due to the incongruous definitions of update in the two realms, their evaluation changes when moving from a static to a dynamic setting. 
To solve this mismatch we look at LPs through the lenses of abstract argumentation. From the LP perspective, adding attackers on the corresponding AF means adding vulnerabilities to the rule whose head identifies the attacked argument.
We call this operation \textit{rule refinement}, defined as follows.
\begin{definition}[rule refinement]
	Let $r$ and $r'$ be two arbitrary LP rules. To express that $r$ is refined by means of $r'$, we use $$\mathtt{refine}(r,r'):=\{id(r): head(r) \gets body(r)\cup body(r').\}.$$ 
\end{definition}

In this section, we introduce a novel notion of strong equivalence for the class of (h-unique) atomic programs based on the rule refinement operator.  

\subsection{Rule Refinement for h-unique LPs}

We first consider the class $\Xi$ of h-unique atomic programs. 
Within this class, each rule can be identified by means of its head-atom. Consequently, we can omit the explicit identifiers from rules and use the head-atoms instead. 
We can now introduce a novel notion of LP update, that we call Rule Refinement (RR, for brevity). 
Updating a program $P$ with a rule $r'$ via RR consists of two possible operations: if $head(r')\notin head(P)$, the update yields a set-union; otherwise, the rule-body of $r\in P$ for which $head(r)=head(r')$ is merged with that of $r'$. 

\begin{definition}[$\uplus$-update]
	Let $P$ be an h-unique atomic program and $r$ a rule in $P$. The $\uplus$-update (or simply update) of $P$ by means of an atomic rule $r'$ is:
	$$ P \uplus r' =
	\begin{cases}
		P\cup \{r'\} &\text{if } head(r')\notin head(P)\\
		P\setminus \{r\}\cup \mathtt{refine}(r,r') &\text{otherwise }
	\end{cases}
	$$
\end{definition}   

\begin{example}
	The $\uplus$-update of $P=\{a\gets \naf b., \; b\gets \naf a., \; c \gets.\}$ via the rule $r=c \gets \naf d.$ is: 
$ P\uplus r=\{a\gets \naf b., \; b\gets \naf a., \; c \gets \naf d.\}$.		
\end{example}

Towards the generalization of $\uplus$ between sets of rules, notice that the consecutive applications of $\uplus$ preserve the result independently of the order of execution. 

\begin{restatable}{lemma}{AssocHeadUni}\label{lemma:assoc1}
	For an atomic program $P$ and two rules $r_1$ and $r_2$ it holds that $(P\uplus r_1) \uplus r_2 = (P\uplus r_2 )\uplus r_1$.
\end{restatable}

Updating an h-unique program $P$ with a program $Q=\{r_1,\dots,r_n\}$ results in $P\uplus Q=P\uplus r_1 \uplus \dots  \uplus r_n$, which is guaranteed to be atomic and h-unique. 

\begin{restatable}{proposition}{RRPreserveHeadUniAtom}
	For an h-unique atomic program $P$ and any atomic program $Q$, the program $P\uplus Q$ is h-unique and atomic.  
\end{restatable}

We can now introduce RR strong equivalence for h-unique atomic programs.

\begin{definition}
	Two h-unique atomic programs $P$ and $Q$ are strongly equivalent under Rule Refinement in $\Xi$, written $P \equiv^\Xi_r Q$, iff for any program $R$: either (1) $P\uplus R\notin \Xi$ and $Q\uplus R\notin \Xi$ or (2) $AS(P\uplus R)=AS(Q\uplus R)$.
\end{definition}
By the definition of RR-update, whenever $R\in \Xi$, condition (1) is never satisfied.  
Moreover, strongly equivalent programs under standard expansions are not guaranteed to be RR strong equivalent in $\Xi$.

\begin{proposition}\label{pro:RuleRefvsStandard}
	For any two h-unique atomic programs $P$ and $Q$, it holds that: (i) $P\equiv^{\Xi}_{s} Q$ does not imply $P\equiv^{\Xi}_{r} Q$; (ii) $P\equiv_{s} Q$ does not imply $P\equiv^{\Xi}_{r} Q$.
\end{proposition}
\begin{example}\label{example:RuleRefvsStandard}
	Consider the LPs $P =\{ a\gets \naf b, \naf c., \; b\gets \naf a, \naf c., \; c \gets.\}$ and $Q =\{ a\gets \naf b, \naf c., \; b\gets \naf c., \; c \gets.\}$. 
	Let $R$ be an (atomic) program. Since $c\gets $ is contained in both $P\cup R$ and $Q\cup R$, $c$ occurs in every answer-set of both programs. Thus, all the other rules in $P$ and $Q$ do not fire, making $P$ and $Q$ strongly equivalent, i.e.\ $P\equiv_{s} Q$ and $P\equiv^{\Xi}_{s} Q$. 
	However, for the program $R'=\{c\gets \naf d., d\gets.\}$, we get: (1) $P\uplus R' \in \Xi$ and $Q\cup R' \in \Xi$ and (2) $AS(P\uplus R')\neq AS(Q\uplus R')$ since $\{a,d\}$ is not a common answer-set of the two. 
\end{example}

\subsection{Rule Refinement for LPs}

In this section we consider a notion of strong equivalence under rule refinement for the whole class $\Lambda$ of atomic LPs by dropping the requirement of h-uniqueness.  
We deal with (possibly non-strict) atomic programs, in which several rules with the same head might occur. Here, the information provided by rule identifiers becomes relevant when considering possible updates: identifiers allow to distinguish among updates that involve rules already contained in $P$ or new ones. Indeed, updating a program with a new rule $r'$ may result in the addition or refinement, irrespective of the head of $r'$. We then adapt the notion of update. 
\begin{definition}[$\squplus$-update]
	Let $P$ be an atomic program and $r$ a rule in $P$. We define the $\squplus$-update (or simply update) of $P$ by means of a new rule $r'$ as:
	$$ P \squplus r' =
	\begin{cases}
		P\cup \{r'\} &\text{if } id(r')\notin id(P)\\
		P\setminus \{r\}\cup \mathtt{refine}(r,r') &\text{otherwise }
	\end{cases}
	$$
\end{definition}

Identifiers guide the update either towards the addition of the new rule or the refinement of an old rule with the same identifier. 

\begin{example}\label{ex: squplus_update}
	The $\squplus$-update of $P=\{1: a\gets., \; 2: b\gets \naf a., \; 3: b \gets.\}$ by means of $r=3: a \gets \naf c.$ is: 
	$P\squplus r=\{1: a\gets., \; 2: b\gets \naf a., \; 3: b \gets \naf c.\}$.
	
\end{example}

As before, to generalize the notion of $\squplus$-update between sets of rules, we note that associativity is preserved. Further, $P\squplus Q$ is atomic whenever $P$ and $Q$ are. 

\begin{restatable}{lemma}{AssocAtomic}\label{lemma:assoc2}
	For an atomic program $P$ and two rules $r_1$ and $r_2$ it holds that $(P\squplus r_1) \squplus r_2 = (P\squplus r_2 )\squplus r_1$.
\end{restatable}
\begin{restatable}{proposition}{RRPreserveAtom}
	For any two atomic programs $P$ and $Q$, $P\squplus Q$ is atomic.
\end{restatable}

We are now able to define strong equivalence under rule refinement.
\begin{definition}
	Two (atomic) programs $P$ and $Q$ are strongly equivalent under Rule Refinement in a class $\Pi$, written $P \equiv^{\Pi}_{r^+} Q$, iff for any program $R$: either (1) $P\squplus R\notin \Pi$ and $Q\squplus R\notin \Pi$ or (2) $AS(P\squplus R)=AS(Q\squplus R)$.
\end{definition}
In the following we will consider $\Pi=\Lambda$. Then, in case $P$ and $Q$ are atomic, condition (1) is false iff $R \in \Lambda$, i.e., it is sufficient to check condition (2) for $R \in \Lambda$. 
For atomic programs,
we inherit results from Proposition \ref{pro:RuleRefvsStandard}, i.e.\ $P\equiv^{\Lambda}_{s} Q$ does not imply $P\equiv^{\Lambda}_{r^+} Q$ (see Example \ref{example:RuleRefvsStandard}).
Further this notion faithfully generalizes RR strong equivalence for h-unique programs. %

\begin{restatable}{proposition}{RplusImpliesR}
	For the class of h-unique atomic programs $\Xi$ and $P,Q \in \Xi$, it holds that $P\equiv^{\Xi}_{r^+}Q$ implies $P\equiv^{\Xi}_{r} Q$.	
\end{restatable}

\section{Rule Refinement Captures Strong Equivalence in Abstract Argumentation }
In the present section, we show that the notion of strong equivalence under Rule Refinement matches the one for abstract argumentation frameworks. 

\subsection{AFs and h-unique Logic Programs}
We first consider h-unique LPs and AF strong equivalence. Inspired by the syntactic characterization of strong equivalence for strict AFs \cite{OikarinenW11}, we define the notion of \textit{kernel} for logic programs. Similarly to its AF counterpart, the ASP kernel of a program is obtained by deleting dispensable vulnerabilities from its rules. Each atom that is in conflict with itself, i.e.\ it occurs in the head as well as in the negative body of a rule, is removed from every other rule-body. We call such rules \emph{loops} and define $loop(P) = \{r \in P \mid head(r)\in neg(r)\}$ for an LP $P$.

\begin{definition}\label{label:asp_kernel}
	Let $loopH(Q)=\{\naf head(r) \mid r \in loop(Q)\}$. The kernel of a h-unique program $P$ is $P^K = \{head(r) \gets body(r) \setminus loopH(P \setminus \{r\}).  \mid r \in P\}$. 
\end{definition}

\begin{example}
	In what follows, we represent a program $P$ and its kernel $P^K$.

	\vspace{1mm}
    \hspace{-15pt}
	\begin{minipage}{.57\textwidth}
		\begin{tabular}{cl}
			$P=$&$\{ a\gets \naf a, \naf b., \; b\gets \naf a, \naf c.,$ \\
			&\; $c \gets \naf c, \naf d., \; d\gets \naf a, \naf c. \}$ 
		\end{tabular}
	\end{minipage}
	\begin{minipage}{.5\textwidth}
		\begin{tabular}{cl}
			$P^K =$&$\{ a\gets \naf a, \naf b., \; b\gets., $ \\
			& \; $ c \gets \naf c, \naf d., \; d\gets. \}$
		\end{tabular}
	\end{minipage}	
\end{example}

As a sanity check, observe that for our notions of LP and AF kernel, the order in which we apply the translation and construct the kernel does not matter.     
\begin{restatable}{proposition}{OrderKernelTransf}\label{pro:kernal-translation}
	For a h-unique atomic program $P$, it holds that $(F_P)^{SK} = F_{P^K}$. Analogously for an AF $F$, it holds that $(P_F)^K = P_{F^{SK}}$.
\end{restatable}

Combining two programs $P$ and $Q$ through Rule Refinement and transforming the resulting program into an AF yields the same framework as taking the AF union of the transformed programs.

\begin{restatable}{lemma}{DistribUnderTransl}\label{lemma:uplus to cup}
	For two h-unique atomic programs $P$ and $Q$ , it holds that $F_{P \uplus Q} = F_P \cup F_Q$. Vice-versa, for two AFs $F$ and $G$, it holds that $P_{F\cup G} = P_F \uplus P_G$.
\end{restatable}

As illustrated for AFs, ASP kernels are equivalent to the original instance.

\begin{restatable}{proposition}{AnswSetKernelEquiv}\label{pro: eq under ASP-kernel}
	For any h-unique atomic program $P$, $AS(P) = AS(P^K)$. 
\end{restatable}
Two programs $P$ and $Q$ maintain the same kernel under any common update.

\begin{restatable}{proposition}{SameKernelUnderUpdate}
	Let $P$ and $Q$ be two h-unique atomic programs with $P^K = Q^K$. Then for any h-unique atomic program $R$, it holds that $(P \uplus R)^K = (Q\uplus R)^K$.
\end{restatable}  
Finally, we show that Rule Refinement characterizes AF strong equivalence. 

\begin{restatable}{theorem}{StrongEquivHuniqueProg}\label{StrongEquivHuniqueProg}
    Let $P$ and $Q$ be two h-unique atomic programs, then the following conditions are equivalent:
	\begin{enumerate}
		\item $P$ and $Q$ have the same kernel, i.e.\ $P_K=Q_K$. 
		\item The AFs $F_P$ and $F_Q$ have the same stable-kernel, i.e.\ $(F_P)^{SK}=(F_Q)^{SK}$.
		\item The AFs $F_P$ and $F_Q$ are strongly equivalent for $\stb$ semantics, i.e.\ $F_P \equiv_s F_Q$.
		\item $P$ and $Q$ are strongly equivalent under Rule Refinement, i.e.\ $P\equiv^{\Xi}_r Q$.	
	\end{enumerate}
\end{restatable}

\subsection{Atomic Programs and Well-Formed CAFs}\label{subsec:SELPsandWFCAFs}

First observe that the ASP-kernel is not helpful to characterize RR strong equivalence for atomic programs, e.g. for $P=\{ 1: a\gets \naf b., \ 2: b\gets ., \ 3: b \gets \naf b.\}$ and $P^K=\{ 1: a\gets., \ 2: b\gets., \ 3: b \gets \naf b.\}$, we get $AS(P)\neq AS(P^K)$. 
We then provide a direct characterization of RR strong equivalence 
and connect it to well-formed CAFs.
Two programs are RR strongly equivalent if their rules' components pairwise coincide, with the possible exception of loop-rules' heads. 
\begin{restatable}{lemma}{LoopRulesRemovedReduct}\label{lemma:LoopRulesRemovedReduct}
	Let $P$ be a program and $r\in loop(P)$ a loop-rule in $P$. Then, $r$ gets removed when computing $P^S$ for any $S\in AS(P)$.
\end{restatable}
\begin{restatable}{theorem}{StrongEquivAtomicProg}\label{the:StrongEquivAtomicProg}
	Two atomic programs $P$ and $Q$ are strongly equivalent under Rule Refinement, denoted $P \equiv^{\Lambda}_{r^+} Q$, iff it holds that:
	\begin{enumerate}
		\item $id(r)=id(r')$ implies $body(r)=body(r')$ for all rules $r\in P$ and $r'\in Q$;
		\item $id(r)=id(r')$ implis $head(r)=head(r')$ for all $r\notin loop(P)$ and $r'\notin loop(Q)$
	\end{enumerate} 
\end{restatable}

As before, we analyze how this new notion of update is understood in terms of CAFs. On the one hand, simply adding a rule to a program $P$ amounts to augmenting $\mathcal{F}_P$ with an argument and possibly new attacks. On the other hand, refining a rule corresponds to adding claim-attacks to $\mathcal{F}_P$.  

\begin{example}\label{ex: squplus_update cont}
	Consider the CAFs $\mathcal{F}_P$ (left), $\mathcal{F}_{r}$ (center) and $\mathcal{F}_{P\squplus r}$ (right) corresponding to $P$, $r$ and $P\squplus r$ from Example \ref{ex: squplus_update}:
	
\hspace{-18pt}
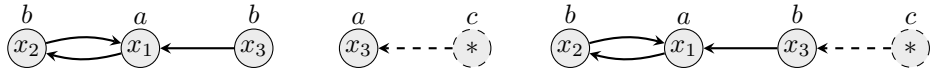
\begin{minipage}{.35\textwidth}
	\begin{tikzpicture}
		\path
		(-1,0.4) node (a) {$a$}
		(-1,0) node[arg](1){$x_1$}
		(0.5,0.45) node (b1) {$b$}
		(0.5,0) node[arg](3){$x_3$}
		(-2.5,0.45) node (b2) {$b$}
		(-2.5,0) node[arg] (2) {$x_2$}
		(0,.6) node (v) {}
		;
		
		\path[->,>=stealth,thick]
		(1) edge[bend left=15] (2)
		(2) edge[bend left=15] (1)
		(3) edge (1)
		;
	\end{tikzpicture}
\end{minipage}
\begin{minipage}{.22\textwidth}
	\begin{tikzpicture}
		\path
		(-1,0.4) node (a) {$a$}
		(-1,0) node[arg](x3){$x_3$}
		(0.5,0.4) node (c) {$c$}
		(0.5,0) node[argd](star){$*$}
		(0,.6) node (v) {}
		;
		
		\path[->,>=stealth,thick]
		(star) edge[dashed] (x3)
		;
	\end{tikzpicture}
\end{minipage}
\begin{minipage}{.4\textwidth}
	\begin{tikzpicture}
		\path
		(-1,0.4) node (a) {$a$}
		(-1,0) node[arg](1){$x_1$}
		(0.5,0.45) node (b1) {$b$}
		(0.5,0) node[arg](3){$x_3$}
		(-2.5,0.45) node (b2) {$b$}
		(-2.5,0) node[arg] (2) {$x_2$}
		(2,0.4) node (c) {$c$}
		(2,0) node[argd](star){$*$}
		(0,.6) node (v) {}
		;
		
		\path[->,>=stealth,thick]
		(1) edge[bend left=15] (2)
		(2) edge[bend left=15] (1)
		(3) edge (1)
		(star) edge[dashed] (3)
		;
	\end{tikzpicture}
\end{minipage}	
\end{example}

We define an update operator for well-formed CAFs that mimics $\squplus$-updates.

\begin{definition}\label{def:wf_CAF_update}
	
	Let $\mathcal{F}_1=(A_1,R_1^{\mathcal{C}},\gamma_1)$ and $\mathcal{F}_2=(A_2,R_2^{\mathcal{C}},\gamma_2)$ be two well-formed CAFs, the update of $\mathcal{F}_1$ via $\mathcal{F}_2$ is $\mathcal{F}_1\cup \mathcal{F}_2=(A_1\cup A_2, R_1^{\mathcal{C}}\cup R_2^{\mathcal{C}}, \gamma_1 \overrightarrow{\times} \gamma_2)$ where:
	$$
	\gamma_1 \overrightarrow{\times} \gamma_2(a)=
	\begin{cases}
		\gamma_1(a) &\text{if } a\in A_1\\
		\gamma_2(a) &\text{if } a\in A_2\setminus A_1
	\end{cases}
	$$
\end{definition}
Updating an atomic program with an atomic rule might result in a violation of compatibility in the corresponding CAFs. For instance, in Example \ref{ex: squplus_update cont} the argument $x_3$ is labelled with two different claims in $\mathcal{F}_P$ and $\mathcal{F}_r$. For such incompatible expansions, we operate a choice between the claim functions, prioritizing $\gamma_1$. 
Thus, when two CAFs are incompatible with respect to some arguments, we choose the original claim as it is done with the rule-head of any rule $r$ that is refined via an $\squplus$-update.
We can now define strong equivalence in this setting. 
\begin{definition}
	Given any two well-formed CAFs $\mathcal{F}$ and $\mathcal{G}$, we say that $\mathcal{F}\equiv_{s} \mathcal{G}$ if and only if $\sigma(\mathcal{F}\cup \mathcal{H})=\sigma(\mathcal{G}\cup \mathcal{H})$ for any well-formed CAF $\mathcal{H}$.
\end{definition}

This notion differs from the original in two respects: (1) it does not restricts to compatible updates only and (2) it requires updates to be well-formed.  
This notion of strong equivalence for CAFs corresponds to strong equivalence under rule refinement for atomic LPs. To show this, we use of the following lemma.

\begin{restatable}{lemma}{UpdateAtomDistribUnderTransl}\label{lemma:CAF_distrib}
	For any atomic programs $P$ and $Q$, it holds that $\mathcal{F}_{P \squplus Q}=\mathcal{F}_P\cup \mathcal{F}_Q$. Conversely, for any two well-formed CAFs $\mathcal{F}$ and $\mathcal{H}$, $P_{\mathcal{F}\cup \mathcal{H}}=P_{\mathcal{F}}\squplus P_{\mathcal{H}}$. 
\end{restatable}

\begin{restatable}{theorem}{StrongEquivLPvsCAF}\label{StrongEquivLPvsCAF}
	For any atomic programs $P$ and $Q$, the following are equivalent:
	\begin{itemize}
		\item $\mathcal{F}_P$ and $\mathcal{F}_Q$ 
        are strongly equivalent for stable semantics, i.e.\ $\mathcal{F}_P\equiv_s \mathcal{F}_Q$.
		\item $P$ and $Q$ are strongly equivalent under Rule Refinement, i.e.\ $P\equiv^{\Lambda}_{r^+}Q$.
	\end{itemize}
\end{restatable}

Thus, characterizing strong equivalence on well-formed CAFs is now possible. 
\begin{corollary}
	For two well-formed CAFs $\mathcal{F}\!=\!(A_\mathcal{F}, R^\mathcal{C}_\mathcal{F}, \gamma_\mathcal{F})$ and $\mathcal{G}\!=\!(A_\mathcal{G}, R^\mathcal{C}_\mathcal{G}, \gamma_\mathcal{G})$, we have $\mathcal{F} \equiv_s \mathcal{G}$ iff (1) $A_{\mathcal{F}} = A_\mathcal{G}$ and $R_{\mathcal{F}}^{\mathcal{C}} = R_{\mathcal{G}}^{\mathcal{C}}$, and (2) for each $x_i \in A$, either
		$\gamma_\mathcal{F}(x_i)=\gamma_\mathcal{G}(x_i)$ or 
		$(\gamma_\mathcal{F}(x_i),x_i) \in R_{\mathcal{F}}^{\mathcal{C}} \land (\gamma_\mathcal{G}(x_i),x_i) \in R_{\mathcal{F}}^{\mathcal{C}}$ holds.
\end{corollary}

Having a characterization for strong equivalence under rule refinement, we now prove that it generalizes standard strong equivalence for atomic programs.
\begin{restatable}{proposition}{RRStrongEquivAtomImpliesStandard}
	For any $P$ and $Q$ atomic programs, $P\equiv^{\Lambda}_{r^+}Q$ implies $P\equiv_{s} Q$.
\end{restatable}

\section{Discussion}
In this paper, we analyzed the correspondence between certain classes of logic programs and abstract argumentation frameworks. In the static setting, we extended existing semantic equivalence results from strict to non-strict programs and AFs, by appealing to the concept of ungrounded attack. Further, we adapted the definition of well-formed CAFs to align precisely with (possibly non-strict) atomic programs. 
We then examined how equivalence in dynamic contexts gets lost in translation, and 
identified the source of such misalignment in the diverging representations of expansions. To overcome the issue, we proposed a new notion of update for (h-unique) atomic programs, called rule refinement. 
We then introduced the notion of strong equivalence under Rule Refinement for (h-unique) atomic programs, investigated its relations with the standard one, and proven that characterizes strong equivalence for AFs and well-formed CAFs. 

The notion of strong equivalence has been extensively investigated in both logic programming and argumentation. The misalignment has been noted before in the work of \cite{BaumannS22}, that shares with ours the motivation of bridging the gap between strong equivalence notions across different formalisms. Indeed, they provide a general semantic framework for strong equivalence that is formally independent of specific formalism, and capable of subsuming both logic programs and Dung-style AFs. In contrast, by remaining on the syntactic level, our work is closer to original characterizations of strong equivalence and addresses the mismatch in a direct way. 
For logic programs, strong equivalence has been characterized using $\mathit{SE}$-models \cite{LifschitzPV01}, which offer a semantic basis for comparing programs under arbitrary expansions. An $\mathit{SE}$-model of a program $P$ is a pair $(X,Y)$ of sets of atoms such that $X\subseteq Y$, $Y \models P$ and $X \models P^Y$. Two logic programs are strongly equivalent iff their $\mathit{SE}$-models coincide. 
Building on this, \cite{DelgrandePW13} defined belief revision operators for LPs and established representation theorems for various program classes, connecting strong equivalence with stability after revisions. Studying the relation between revision and update operators in ASP \cite{LeiteS23} and our notion of rule refinement is subject of future investigations. 
Moreover, we aim at providing a characterization of RR strong equivalence for the entire class of normal logic programs, and connect this to possibly more expressive types of AFs. We anticipate that this could be achieved by encoding positive dependencies among atoms via some notion of support in the corresponding framework~\cite{CayrolL05a}. To this end, a comparative analysis of our notion of strong equivalence and those defined over argumentation formalisms related to normal LPs, such as abstract dialectical frameworks, is needed~\cite{HeyninckKRST23,RapbergerU23}. 
 \bibliographystyle{splncs04}

\clearpage
\appendix

\section{Omitted Proofs}
We make use of the following lemma for the proof of the Proposition~\ref{prop:equiv LP ill-formed LP} below.

\begin{restatable}{lemma}{QuasiVulNotAS}\label{lem:strictlp}
	For any atomic program $P$ and atom $c\in neg(P)$ such that $c\notin head(P)$, it holds that $c\notin S$ for any answer-set $S\in AS(P)$. 
\end{restatable}
\begin{proof}
	We prove the statement by contradiction. Suppose that $c$ is contained in an answer-set $S$ of $P$.
	Then $c$ has to be in the minimal model of $P^S$. Hence, there is a rule in $P^S$ that derives $c$, in contradiction with our hypothesis. 
\end{proof}

\ASEqualPToWellFormP*
\begin{proof} 
	Consider the \textit{ungrounded vulnerabilities} $UV(P)=\{\naf b \mid b\in neg(P) \setminus head(P)\}$ of $P$.
	In the case that $UV(P)=\emptyset$, then $P=P_{\downarrow head}$ and we are done. Thus, for the remainder of the proof, we will assume
	$UV(P)\not=\emptyset$.
	
	$AS(P)\subseteq AS(P_{\downarrow head})$:
	We show $P^S=P^S_{\downarrow head}$ for any set $S\in AS(P)$. From Lemma~\ref{lem:strictlp}, we know that $c\notin S$ for any $\naf c\in UV(P)$. Hence, every occurrence of $\naf c$ gets removed when computing $P^S$. Since $P$ and $P_{\downarrow head}$ are identical except $UV(P)$, we derive that $P^S=P^S_{\downarrow head}$, and consequently that $S\in AS(P_{\downarrow head})$.
	
	$AS(P)\supseteq AS(P_{\downarrow head})$:
	Take any $S'\in AS(P_{\downarrow head})$. By definition of $P_{\downarrow head}$, we know that $c\notin S$ for any $\naf c\in UV(P)$. Similarly to (i), we derive that $P^{S'}=P^{S'}_{\downarrow head}$, and consequently that $S'\in AS(P)$.
\end{proof}

\OrderTranslQuasi*
\begin{proof}
	For any program $P$ we fix a set of atoms $C\subseteq neg(P)$ that do not occur in $P_{\downarrow head}$ (i.e. $C\cap head(P)=\emptyset$). By Definition~\ref{def:LP2AF} for every $c\in C$ there is at least one conflict $(c,a)\in R_P$ where $a\in head(P)$. Since $c\notin head(P)$, the corresponding $c$ is not contained in $A_P$ and $(c,a)$ is an ungrounded attack. Hence, by computing $P_{\downarrow head}$ we remove every body-atom which induces an ungrounded attack after transformation. That is, $F_{P_{\downarrow head}}=(F_P)_{\downarrow A_P}$.
	
	Let us now consider an AF $F$. By Definition~\ref{ex:AF2LP}, any $a\in \mathcal{U}\setminus A$ corresponds to an atom with the same name that (negatively) occurs in $P_F$ but not in $P_{F_{\downarrow A}}$. Moreover, since $a\notin A$, we have that $a\notin head(P_F)$. Therefore, any such $a$ is removed when computing $(P_F)_{\downarrow head}$. Thus, $P_{F_{\downarrow A}}=(P_F)_{\downarrow head}$. 
\end{proof}

\SemCorrAFLP*
\begin{proof}
	From propositions and lemmas above, it can be easily constructed the following chain of equivalences: $AS(P)=AS(P_{\downarrow head})=\stb(F_{P_{\downarrow head}})=\stb((F_P)_{\downarrow A_P})=\stb(F_P)$. In a similar way, it is easily proven that $\stb(F)=\stb(F_{\downarrow A}) = AS(P_{F_{\downarrow A}})= AS((P_F)_{\downarrow head})= AS(P_F)$. 
\end{proof}

\AssocHeadUni*
\begin{proof}
	For any atomic program $P$ and rules $r_1$ and $r_2$ there are two cases: 
	\begin{itemize}
		\item $head(r_1)=head(r_2)$. Assume $head(r_1)\in head(P)$. Then there is a rule $r\in P$ that is removed and refined twice, generating $r': head(r) \gets body(r) \cup body(r_1) \cup body(r_2)$. Since the set-union $\cup$ already satisfy associativity, the property is thus inherited by $\uplus$. Assume $head(r_1)\notin head(P)$. Then we need to prove $(P\cup r_1) \uplus r_2 = (P\cup r_2 )\uplus r_1$. Considering both updates in parallel, $r_1$ (resp. $r_2$) is removed and refined via $r_2$ (resp. $r_1$). Since $head(r_1)=head(r_2)$, it follows immediately that $\mathtt{refine}(r_1,r_2)=\mathtt{refine}(r_2,r_1)$, as desired. 
		\item $head(r_1)\neq head(r_2)$. Again we consider two cases: (1) No rule $r\in P$ has the same head as $r_1$ or $r_2$. In this case, we have a simple set-union, which is already associative. (2) One rule $r\in P$ gets removed and refined via $r_1$ (resp. $r_2$) and $r_2$ (resp. $r_1$) is simply added. Clearly, the order in which this two operations are performed is irrelevant. 
	\end{itemize}
\end{proof}

\RRPreserveHeadUniAtom*
\begin{proof}
	We prove the statement by contradiction. Assume $(P\uplus Q)$ is either not h-unique or not atomic. In the first case, we derive that there are two rules $r$, $r'$ in  $(P\uplus Q)$ such that $head(r)=head(r')$.  As both $P$ and $Q$ are h-unique we have
	$r \in P$ and $r'' \in Q$.
	But then by the definition of update, we know that $r$ and $r'$ are both contained in $P$ or in $Q$. However, this contradicts our hypothesis. In the second case, there is one rule $r\in P\uplus Q$ which is not atomic. Again, by definition of update this can happen only when there is a non-atomic rule in either $P$ or $Q$, in contradiction with our hypothesis.
\end{proof}

\AssocAtomic*
\begin{proof}
	The proof can be inherited from Lemma~\ref{lemma:assoc1} by replacing $head(r)$ and $head(r')$ with $id(r)$ and $id(r')$ respectively.
\end{proof}

\RRPreserveAtom*
\begin{proof}
	The statement follows from the definition of $\squplus$, since both set-union and rule refinement via some rule $r\in Q$ yield a program $P\squplus r$ where no positive atom occur in $body(P\squplus r)$.  
\end{proof}

\RplusImpliesR*
\begin{proof}
	Towards a contradiction assume that $P\equiv^{\Xi}_{r^+}Q$ but there is a $R \in \Xi$ such that $P \uplus R \not\equiv Q \uplus R$. In order to use $R$ with $\squplus$ we assign to each rule of $R$ an identifier, such that rules get the same identifier iff they have the same head atom. The we have that $P \uplus R = P\squplus R$ and $Q \uplus R = Q\squplus R$. We thus get $P \squplus R \not\equiv Q \squplus R$, which is in contradiction with our assumption that $P\equiv^{\Xi}_{r^+}Q$.
\end{proof}

\OrderKernelTransf*
\begin{proof}
	To prove $(F_P)^{SK} = F_{P^K}$ we need to show (i) $(A_P)^{SK}= A_{P^K}$ and (ii) $(R_P)^{SK}= R_{P^K}$. Since $head(P)=head(P^K)$, we know that $A_P=A_{P^K}$ by Definition~\ref{def:LP2AF}. Moreover, by definition of AF stable-kernel, $A=A^{SK}$, from which we derive $(A_P)^{SK}= A_{P^K}$. In order to prove (ii), we need to show that each vulnerability $\naf a$ removed when computing $P^K$ exactly corresponds to an attack $(a,x)\in R_P$ such that $(a,a)\in R_P$ with $a\neq x$. By Definition~\ref{def:LP2AF}, we derive that: (1) each $r'\in loop(P)$ corresponds to a self-attack $(a,a)\in R_P$ with $head(r')=a$; (2) for all $r\neq r'$, each atom $a \in neg(r)$ identifies an outgoing attack from a self-attacking argument $a\in A_P$ towards an argument $x=head(r)$ with $x\neq a$. Thus, for each $r\in P$ and $r'\in loop(P)$, $loopH(P \setminus \{r\})=head(r')$.  
	As a result, when computing the ASP kernel $P^K$, we obtain $R_{P^K}=R_P\setminus \{(a,x)\mid (a,a)\in R_P, a\neq x\}=(R_P)^{SK}$. \medskip
	
	We now turn to show $(P_F)^K = P_{F^{SK}}$. We need to show that for each rule $r\in (P_F)^K$ there is exactly one rule $r' \in P_{F^{SK}}$ such that (i) $head(r) = head(r')$ and (ii) $body(r) = body(r')$.
	We can show (i). There is exactly one rule per argument in $F$, whose head is not modified or removed when computing the kernel. Moreover, no argument in $F$ gets removed when obtaining $F^{SK}$. Hence, for each argument $a\in A=A^{SK}$, there is exactly one rule in $P_{F^{SK}}$ whose head is the same as for $(P_F)^K$.
	Regarding (ii), note that by Definition~\ref{def: AF2LP}, we have $P_F = \{a \gets \naf b_1, \dots, \naf b_k \mid a \in A, (b_i, a)\in R \}$. Moreover, we obtain $(P_F)^K=\{head(r) \gets  body(r) \setminus  loopH(P \setminus \{r\})\}$ by definition of ASP kernel. For each rule $r\in P_F$, the set of atoms $\{loopH(P \setminus \{r\})\cap body(r)\}$ corresponds by Definition~\ref{def: AF2LP} to a set of arguments $\{b_1,\dots,b_m\}$ such that $(b_i,b_i)\in R$ and $(b_i,a)\in R$ with $a=head(r)$. Hence, for each $r\in P_F$, computing $body(r) \setminus loopH(P \setminus \{r\})$ amounts to remove every attack $(b_i,a)\in R$. Thus $(P_F)^K = P_{F\setminus \{(b_i,a)\}}$ with $(b_i,b_i)\in R$ and $(b_i,a)\in R$. Finally, we are able to derive  $(P_F)^K = P_{F^{SK}}$. 
\end{proof}

\DistribUnderTransl*
\begin{proof}
	For the first part, observe that $F_{P \uplus Q}$ is mapped to exactly one program $P \uplus Q=P'$. In order to prove the statement, we only need to show that $P'$  corresponds exactly to the framework $F'=F_P\cup F_Q$. We consider arguments and conflicts of $F'$ separately. For arguments, notice that $head(P')=head(P)\cup head(Q)$ by definition of update. Thus, $A'=A_P\cup A_Q$ holds under translation (Definition~\ref{def:LP2AF}). It only remains to prove that conflict relation is preserved from $F_{P\uplus Q}$ to $F_P\cup F_Q$. For this, we show that every rule $r\in P'$ and pair of atoms $(a,b)$ such that $a\in neg(r)$ and $b=head(r)$, there is exactly one conflict $(a,b)\in R'$. Notice first that for each rule $r\in P$ and pair $(a,b)$ occurring in $r$, there is exactly one conflict $(a,b)\in F_P$. For each rule $r'\in Q$ we have two cases: either $head(r')\in head(P)$ or $head(r')\notin head(P)$. Assume $head(r')\in head(P)$. Then, via $\mathtt{refine}(r,r')$ we add a pair $(a',b)$ in $P\uplus r'$ for every vulnerability $a'\in neg(r')$. By applying the translation to the AF, we obtain conflicts $R_P\cup R_{r'}$ (thus, $F_P\cup F_{r'}$). Assume now $head(r')\notin head(P)$. In this case, we simply take $P\cup r'$, which yields $F_P\cup F_{r'}$ under translation. By associativity of $\cup$, it is possible to construct $F_P\cup F_Q$ for all rules in $Q$, concluding the proof. For two AFs $F$ and $G$, the proof that $P_{F\cup G} = P_F \uplus P_G$ is analogous.  
\end{proof}

\AnswSetKernelEquiv*
\begin{proof}
	In order to show that the statement holds, we prove that for any $S$ answer set of $P$, $P^S=(P^K)^S$.  Given that $P^S$ and $(P^K)^S$ are positive and $P$ is atomic, it follows that $body(P^S)=body((P^K)^S)=\emptyset$. Towards contradiction, assume $P^S\neq(P^K)^S$. This can happen only if there are at least two rules $r\in P$ and $r'\in P^K$ with $head(r)=head(r')$ such that: (1) $r'$ is deleted when computing $(P^K)^S$, but $r$ is not deleted when computing $P^S$; or (2) $r$ is deleted when computing $P^S$, but $r'$ is not deleted when computing $(P^K)^S$. First, assume (1) is the case. Since $body(r')\subseteq body(r)$ by definition of ASP kernel, $r$ gets deleted whenever $r'$ is. Indeed, the negated atom that triggers the deletion of $r$ is contained in the body of $r'$. Contradiction. Assume now (2) is true. We know that $body(r')= body(r)\setminus loopH(P\setminus \{r\})$. Therefore, together with (2) we derive that $loopH(P\setminus \{r\})\cap S\neq \emptyset$. However, from the definition of answer-set, we know that for every $r\in loop(P)$, $head(r)\notin S$. Otherwise, the unique rule $r$ deriving $head(r)$ would be deleted when building the reduct, making $S$ non minimally closed under $P^S$. Again, we reached a contradiction.
	The argument is analogous when we consider any answer set $S'$ of $P^K$. To derive $P^{S'}=(P^K)^{S'}$ it is sufficient to replace each occurrence of $S$ with $S'$. 
\end{proof}

\SameKernelUnderUpdate*
\begin{proof}
	To prove the statement we show that for each pair of rules $r\in (P \uplus R)^K$ and $r'\in (Q \uplus R)^K$, $head(r)=head(r')$ implies $body(r)=body(r')$. First observe that $P^K = Q^K$ implies that for each rule in $P$ there is another rule in $Q$ such that $head(r)=head(r')$. Since updates and the kernel do not change rule-heads, the same holds for each pair of rules $r\in (P \uplus R)^K$ and $r'\in (Q \uplus R)^K$. It remains to prove that for such rules $r$, $r'$ it holds that $body(r)=body(r')$. Let $r\in (P \uplus R)^K$  with $a\in body(r)$ and $b=head(r)$. We show that $a\in body(r')$ for the rule $r'\in (Q \uplus R)^K$ with $head(r')=b$. Since $a\in body(r)$, we can derive that there is no loop-rule $r_l\in loop(P\uplus R)$ with $head(r_l)=a$ (by definition of kernel). Thus, $r_l\notin P^K$ and $r_l\notin R^K$. Moreover, $r_l\notin Q^K$ since $P^K = Q^K$. Hence, since $r\in (P \uplus R)^K$, there is a rule $r^\star$ with $a\in body(r^\star)$ and $b=head(r^\star)$ in $P^K$ or $R^K$. We proceed by cases. Assume $r^\star \in P^K$. Given that $P^K = Q^K$, $r^\star \in Q^K$. Further, since there is no loop-rule $r_l\in loop(R^K)$ with $head(r_l)=a$, we derive that $a\in body(r')$ for the rule $r'\in (Q \uplus R)^K$ with $head(r')=b$. Assume now that $r^\star \in R^K$. Then there is no loop-rule $r_l\in loop(Q^K)$ with $head(r_l)=a$. Finally, it follows that $r_l\notin loop(Q\uplus R)$, and thus $a\in body(r')$ for the rule $r'\in (Q \uplus R)^K$ with $head(r')=b$.
\end{proof}

\StrongEquivHuniqueProg*
\begin{proof}
We initially show the equivalence from 1.\ to 4.\ and then the other direction. 
From $P_K=Q_K$ we derive $F_{P^K}=F_{Q^K}$ via Definition~\ref{def:LP2AF}. Further, via Proposition~\ref{pro:kernal-translation} we obtain $(F_P)^{SK}=(F_Q)^{SK}$. This is equivalent to $F_P\equiv_{s} F_Q$ from the characterization of strong equivalence for AFs. From the definition of strong equivalence, we derive that for any AF $H$, $\stb(F_P\cup H)=\stb(F_Q\cup H)$ and $AS(P_{F_P\cup H}=AS(P_{F_Q\cup H})$ by Proposition~\ref{prop:equiv AF LP}. Now, given Lemma~\ref{lemma:uplus to cup}, we can rewrite the previous equivalence as $AS(P_{F_P}\uplus P_H)=AS(P_{F_Q}\uplus P_H)$. Since the transformation between AFs and h-unique atomic programs is an isomorphism (Remark \ref{rem: iso quasi-AF LP}), it holds that $AS(P\uplus P_H)=AS(Q\uplus P_H)$. Therefore, for any program $H'=P_H$ associated to the AF $H$ we get $AS(P\uplus H')=AS(Q\uplus H')$. Since $H$ is an AF, $P_H$ is guaranteed to be h-unique. Further, since $P$ and $Q$ are also h-unique by hypothesis, we derive $P\uplus H'\in \Xi$ and $Q\uplus H'\in \Xi$, i.e. $P\equiv^{\Xi}_{r} Q$.

We now prove the other direction. By definition, $P\equiv_{r}Q$ entails $AS(P\uplus R)=AS(Q\uplus R)$ for any h-unique atomic program $R$. Thus under translation we get $F_{P\uplus R}$ and $F_{Q\uplus R}$ s.t. $\stb(F_{P\uplus R})=\stb(F_{Q\uplus R})$. Then, via Lemma~\ref{lemma:uplus to cup}, we rewrite the previous equivalence as $\stb(F_P \cup F_R)=\stb(F_Q\cup F_R)$ for any program $R$. Let us call $H$ the unique AF corresponding to each $F_R$, we write $\stb(F_P \cup H)=\stb(F_Q\cup H)$ for any $H$. Thus, $F_P \equiv_s F_Q$, by definition of strong equivalence. From Proposition \ref{pro:AF_SE} the two frameworks have the same stable kernel: $(F_P)^{SK}=(F_Q)^{SK}$. We transform back to programs, obtaining $P_{(F_P)^{SK}}=P_{(F_Q)^{SK}}$ and $(P_{F_P})^K=(P_{F_Q})^K$ via Proposition \ref{pro:kernal-translation}. Finally, Remark \ref{rem: iso quasi-AF LP} brings us to $P^K=Q^K$. 
\end{proof}

\LoopRulesRemovedReduct*
\begin{proof}
	Towards contradiction, assume there is an answer-set $S\in AS(P)$ such that the modified version of $r$, $id(r): head(r) \gets $, occurs in $P^S$. Since $S$ is closed under $P^S$, we derive that (1) $head(r)\in S$. Moreover, since $r \in loop(P)$, we know that (2) $head(r)\in neg(r)$. Together, (1) and (2) imply that $r$ gets removed when computing $P^S$, by definition of reduct. Therefore, for any answer-set $S\in AS(P)$, $r$ is removed when $P^S$ is created. Contradiction. 
\end{proof}

\StrongEquivAtomicProg*

\begin{proof}
	
	($\Leftarrow$) 1.\ and 2.\ together imply that $P$ and $Q$ are identical, except for some loop-rules $\{r_1,\dots, r_n\}\in loop(P)$ and$\{r'_1,\dots, r'_n\}\in loop(Q)$. For $1\leq i\leq n$ and each pair of rules $r_i$, $r'_i$ it holds that $id(r_i)=id(r'_i)$, $body(r_i)=body(r'_i)$, but $head(r_i)\neq head(r'_i)$. To prove $P \equiv^{\Lambda}_{r^+} Q$ we need to show that for any $R$, $P\squplus R \equiv Q\squplus R$, i.e.\ a set of atoms $S$ is an answer-set of $P\squplus R$ iff it is an answer-set of $Q\squplus R$. In fact, if $P\squplus R\notin \Lambda$ and $Q\squplus R\notin \Lambda$, the program $R$ is not atomic and the programs  $P\squplus R$ and  $Q\squplus R$ are outside of the scope of our investigation. 
	Consider, then, the case where $R$ is atomic. For any pair of rules $r_i\in P$ and $r'_i\in Q$, we have two cases for each rule $r\in R$: either (a) $id(r)=id(r_i)$ or (b) $id(r)\neq id(r_i)$. Assume (a) is the case. Thus, $r_i$ and $r'_i$ get refined via $r$ in $P$ and $Q$ respectively, obtaining the rule $r_\star\in P\squplus R$ with $head(r_{\star})=head(r_i)$ and $body(r_{\star})=body(r_i)\cup body(r)$. Similarly, we obtain a rule $r'_{\star}\in Q\squplus R$ with $head(r'_{\star})=head(r'_i)$ and $body(r'_{\star})=body(r_{\star})$. Obviously, $r_{\star}\in loop(P\squplus R)$ and $r'_{\star}\in loop(Q\squplus R)$, since rule refinement only augments the rule-body. Assume $S$ is an answer-set of $P\squplus R$, i.e.\ it is the minimal set closed under $(P\squplus R)^S$. Since $r_{\star}$ is a loop-rule ($head(r_{\star})\in neg(r_{\star})$), we know that the rule $r_{\star}$ with $head(r_{\star})$ gets removed when computing the reduct $(P\squplus r_{\star})^S$. The same happens to $r'_{\star}$ for $(Q\squplus r'_{\star})^S$. Thus, $(P\squplus r_{\star})^S=(Q\squplus r'_{\star})^S$, concluding  $P\squplus r_{\star} \equiv Q\squplus r'_{\star}$. Assume now (b) is the case. Hence, $P\squplus r=P\cup r$ and $Q\squplus r=Q\cup r$. Again, we show that an answer-set of $P\cup r$ is an answer-set of $Q\cup r$. From Lemma \ref{lemma:LoopRulesRemovedReduct}, it follows that $((P\cup r)\setminus loop(P))^S=(P\cup r)^S$ and $((Q\cup r)\setminus loop(Q))^S=(Q\cup r)^S$. Moreover, it holds that $(P\cup r)\setminus loop(P)=(Q\cup r)\setminus loop(Q)$, by hypothesis. Thus we can conclude that $(P\cup r)^S=(Q\cup r)^S$, i.e.\ $P\cup r\equiv Q\cup r$ for any $r$. Since both cases (a) and (b) leads to equivalent updated programs, we can generalize this to any set $R$ of rules by associativity of $\squplus$.
	
	($\Rightarrow$) Suppose now that $P \equiv^{\Lambda}_{r^+} Q$. We derive 1.\ and 2.\  by contradiction. Assume $body(r')\neq body(r)$ for some $r\in P$ and $r'\in Q$ with $id(r)=id(r')$. There are two cases: (i) $body(r')\supset body(r)$ and (ii) $body(r)\supset body(r')$.
	Consider case (i). There is one atom $b_i\in body(r')$ such that $b_i\notin body(r)$. By adding $R=\{i: b_i \gets\}$ to both $P$ and $Q$ (with $i\notin id(P)$) we obtain the programs $P\squplus R$ and $Q\squplus R$ such that $b_i\in S$ for any $S\in AS(P\squplus R)$ and $S\in AS(Q\squplus R)$. Then, for any such $S$, the rule $r$ gets removed when computing $(P\squplus R)^S$. However, $r'$ is not removed for $(Q\squplus R)^S$. Hence, $(P\squplus R)^S\neq (Q\squplus R)^S$ and $P\squplus R \not \equiv Q\squplus R$. Contradiction. The case (ii) can be proven in a similar way by exchanging $r$ and $r'$. We have proven that 1.\ follows from our hypothesis. Now consider 2. Assume $head(r')\neq head(r)$ for some $r\in loop(P)$ and $r'\in loop(Q)$ with $id(r)=id(r')$. From 1.\ we know that $
	body(r)=body(r')$. Hence, for any answer-set $S\in AS(P)$ such that $body(r)\cap S=\emptyset$, we have $head(r)\in S$ (since $id(r): head(r)\gets $ occurs in $P^S$). Similarly, $head(r')\in S'$ for any answer-set $S'\in AS(Q)$ with $body(r')\cap S'=\emptyset$ (since $id(r'): head(r')\gets $ occurs in $Q^S$). Therefore, $P$ and $Q$ are not equivalent. Contradiction. 
\end{proof}

\UpdateAtomDistribUnderTransl*
\begin{proof}
	For the first part, observe that $F_{P \squplus Q}$ is mapped to exactly one program $P \squplus Q=W$. In order to prove the statement, we only need to show that $W$  corresponds exactly to the framework $\mathcal{F}_W=\mathcal{F}_P\cup \mathcal{F}_Q$. We consider arguments, claim-attacks and the claim-functions of $\mathcal{F}_W$ separately. For arguments, notice that $id(W)=id(P)\cup id(Q)$, by definition of $\squplus$-update. Each rule $r\in W$ identified by $id(r)=i$ is mapped to one argument $x_i \in \mathcal{F}_W$, so that $A_{W}=A_P \cup A_Q$. 
	Consider now claim-attacks in $R^{\mathcal{C}}_W$. Each vulnerability $\naf c$ of the rule $r$ with $id(r)=i$ is mapped to the claim-attack $(c,x_i)$ under translation. Since $neg(W)=neg(P)\cup neg(Q)$ (by definition of $\squplus$-update), then such vulnerabilities are mapped to the corresponding set of claim-attacks $R^{\mathcal{C}}_W=R^{\mathcal{C}}_P\cup R^{\mathcal{C}}_Q$. Last, we consider the claim-function $\gamma_W$. We need to prove that $\gamma_W(x_i)=\gamma_P \overrightarrow{\times} \gamma_Q(x_i)$ for all $x_i\in A_W$. Let $Q'\subseteq Q$ be the set of rules of $Q$ such that $id(r)\in id(P)$. Each such rule $r$ refines the rule $r'\in P$ with $id(r')=id(r)$. Therefore, $head(W)=head(P)\cup head(Q\setminus Q')$ 
	by definition of $\squplus$-update.  Moreover, under translation every head atom identifies a claim for the argument $x_i$ corresponding to the rule $r$ with $id(r)=i$. Thereby, the rules in $P$, $Q$ and $Q'$ correspond to the arguments in $A_P$, $A_Q$ and $A_{Q'}=A_P\cap A_Q$, deriving $\gamma_W(A_W)=\gamma_P(A_P)\cup \gamma_Q(A_Q\setminus A_{Q'})=\gamma_P(A_P)\cup \gamma_Q(A_Q\setminus A_P)$.
	For all $x_i\in A_W$, 
	(i) if $x_i\in A_P$, then $\gamma_W(x_i)=\gamma_P(x_i)$ and (ii) if $x_i\in A_Q\setminus A_P$, then $\gamma_W(x_i)=\gamma_Q(x_i)$. By Definition \ref{def:wf_CAF_update}, this is equivalent to $\gamma_W(x_i)=\gamma_P \overrightarrow{\times} \gamma_Q(x_i)$ for all $x_i\in A_W$. For two well-formed CAFs $\mathcal{F}$ and $\mathcal{H}$, the proof that $P_{\mathcal{F}\cup \mathcal{H}} = P_{\mathcal{F}} \uplus P_{\mathcal{H}}$ is analogous. 
\end{proof}

\StrongEquivLPvsCAF*
\begin{proof}
	We prove the two statements from top to bottom and viceversa. Assume $\mathcal{F}_P\equiv_s \mathcal{F}_Q$. By definition, for any well-formed CAF $\mathcal{H}$,  $\stb(\mathcal{F}_P\cup \mathcal{H})=\stb(\mathcal{F}_Q\cup \mathcal{H})$. From Proposition \ref{pro:CAF_LP_Corresp}, we can rewrite the previous as $AS(P_{\mathcal{F}_P\cup \mathcal{H}})=AS(P_{\mathcal{F}_Q\cup \mathcal{H}})$. Thanks to Lemma \ref{lemma:CAF_distrib}, we derive $AS(P_{\mathcal{F}_P}\squplus P_{\mathcal{H}})=AS(P_{\mathcal{F}_Q}\squplus P_{\mathcal{H}})$. We transform this via Proposition \ref{pro:isoCAF} into $AS(P\squplus P_{\mathcal{H}})=AS(Q\squplus P_{\mathcal{H}})$. Thus, for any $H'=P_H$, we get $P\squplus H'=Q\squplus H'$. Since $\mathcal{H}$ is a well-formed CAF, $P_H$ is guaranteed to be atomic. Further, since $P$ and $Q$ are atomic, we derive $P\squplus H'\in \Lambda$ and $Q\squplus H'\in \Lambda$, concluding $P\equiv^{\Lambda}_{r^+}Q$.
	the isomorphism in
    
	Assume now $P\equiv^{\Lambda}_{r^+}Q$. Hence, either (i) $AS(P\squplus R)=AS(Q\squplus R)$ for any program $R$ or (ii) $P\squplus R\notin \Lambda$ and $Q\squplus R\notin \Lambda$. If (ii) is the case, then $R$ is not atomic. If (i) is the case, then we can translate the corresponding programs, deriving $\stb(\mathcal{F}_{P\squplus R})=\stb(\mathcal{F}_{Q\squplus R})$ by Proposition \ref{pro:CAF_LP_Corresp}. Through Lemma \ref{lemma:CAF_distrib}, we then obtain $\stb(\mathcal{F}_P \cup \mathcal{F}_R)=\stb(\mathcal{F}_Q \cup \mathcal{F}_R)$ for any well-formed CAF $\mathcal{F}_R$. We fix the well-formed CAF $\mathcal{H}=\mathcal{F}_R$ and rewrite the previous equivalence as $\stb(\mathcal{F}_P \cup \mathcal{H})=\stb(\mathcal{F}_Q \cup \mathcal{H})$, which  by definition means $\mathcal{F}_P\equiv_s \mathcal{F}_Q$.
\end{proof}

\RRStrongEquivAtomImpliesStandard*
\begin{proof} 
	From the hypothesis, we derive that $P$ and $Q$ differ only with respect to the heads of loop-rules with same identifier (Theorem \ref{the:StrongEquivAtomicProg}).
	Therefore, for any $R$, if we simply add $R$ to $P$ and $Q$ respectively, we will get two programs $P\cup R$ and $Q\cup R$ which themselves differ only w.r.t. heads of loop-rules with same identifier. From Lemma \ref{lemma:LoopRulesRemovedReduct}, these rules are removed when computing the reducts $(P\cup R)^S$ and $(Q\cup R)^S$, allowing us to derive $(P\cup R)^S=(Q\cup R)^S$. 
\end{proof}
\end{document}